\def\year{2018}\relax
\documentclass[letterpaper]{article}

\usepackage{aaai18}  							
\usepackage{times}  							
\usepackage{helvet}  							
\usepackage{courier}  						
\usepackage{url}  								
\usepackage{graphicx}  						
\usepackage{mwe}
\usepackage{multicol}
\usepackage[T1]{fontenc}
\usepackage[utf8]{inputenc}
\usepackage{amssymb}
\usepackage{booktabs}
\usepackage{xr}

\usepackage{algorithm}
\usepackage{algorithmic}

\usepackage{enumitem}
\usepackage[nomessages]{fp}

\usepackage{amsthm}
\theoremstyle{plain}

\newtheorem{lemma}{Lemma}

\newtheorem{proposition}{Proposition}
\theoremstyle{remark}
\newtheorem{remark}{Remark}

\theoremstyle{definition}
\newtheorem{definition}{Definition}

\newtheorem{problem}{Problem}

\usepackage{amsmath}
\usepackage{amsfonts}
\usepackage{dsfont}
\usepackage{upgreek}
\usepackage{bbm}
\usepackage{bbold}
\usepackage{mathbbol}
	
\usepackage{flushend}
\usepackage{xspace}
\usepackage[scaled=.8]{beramono}

\usepackage{url}
\usepackage{hyperref}

\usepackage[round, authoryear]{natbib}

\renewcommand\footnotemark{*}

\usepackage{graphics}
\usepackage{epstopdf}
\usepackage{tikz}
\usepackage{float}
\usepackage{subfigure}
\usepackage{caption}
\graphicspath{ {./figures/} } 

\newcommand{\Sec}[1]		{Sec.\,\ref{#1}}
\newcommand{\Fig}[1]		{Fig.\,\ref{#1}}
\newcommand{\Eq}[1]			{Eq.\,\ref{#1}}
\newcommand{\Tab}[1]		{Tab.\,\ref{#1}}
\newcommand{\Alg}[1]		{Alg.\,\ref{#1}}

\newcommand{\Proposition}[1]{Proposition~\ref{#1}}

\newcommand{\Definition}[1]{Definition~\ref{#1}}

\newcommand{\Problem}[1]{Problem~\ref{#1}}
\newcommand{\ie}   			{i.e.\ }
\newcommand{\etc}   		{etc}
\newcommand{\eg}   			{e.g.\ }
\newcommand{\wrt}   		{w.r.t.\ }
\newcommand{\st}   			{\mbox{s.t.\ }}
\newcommand{\vs}        {vs.\ }
\newcommand{\one}       {\mathbf{1}}

\newcommand{\ind}       {\mathbb{1}}

\newcommand{\real}      {\mathbbm{R}}

\newcommand{\op}[1]     {\,{#1}\,}
\renewcommand*{\top}		{{\mkern-1.5mu\mathsf{T}}}

\renewcommand{\dots}    {...}

\DeclareMathOperator*{\argmin}{\mbox{argmin}}
\DeclareMathOperator*{\elmult}{\odot}

\newcommand{\beforecaptvskip} {\vskip -0.07in}

\newcommand{\inlinetitle}[2]  {\noindent\textbf{\emph{#1}{#2}}}

\newcommand{\myEqCounter}   {\addtocounter{equation}{+1}\hfill(\arabic{equation})}

\newcommand{\HazMat}   {\mathcal{H}}
\newcommand{\HazSpec}  {\rho_H}

\newcommand{\FunMat}   {\mathcal{F}}
\newcommand{\FunSet}   {\mathbbm{F}}
\newcommand{\CTIC}     {\mathcal{CTIC}(\FunMat)}
\newcommand{\CTICnoF}  {\mathcal{CTIC}}
\newcommand{\DeltaF}   {\Delta}
\newcommand{\budget}   {k}
\newcommand{\dset}[1]  {\texttt{#1}}

\newcommand{\methodName}   {NetShape\xspace}

\title{A Spectral Method for Activity Shaping in Continuous-Time Information Cascades}

\nocopyright

\begin{document}

\title{A Spectral Method for Activity Shaping \\in Continuous-Time Information Cascades}

\author{
Kevin Scaman$^{\thanks{\!\!\!\!\!\!\!\!\!\!$^*$~Part of the work has been conducted while author was at CMLA$^1$.},1,2}$
\hspace{1.em}
Argyris Kalogeratos$^{1}$
\hspace{1.5em}
Luca Corinzia$^{\footnotemark,1,3}$
\hspace{1.5em}
Nicolas Vayatis$^{1}$
\\$^1$ CMLA {--} ENS Cachan, CNRS, Universit\'e Paris-Saclay, France
\\$^2$ MSR-Inria Joint Center, 91120 Palaiseau, France
\\$^3$ ETH Zürich, 8092 Zürich, Switzerland
\\\texttt{\{scaman,\,kalogeratos,\,vayatis\}@cmla.ens-cachan.fr},\ \ \texttt{lucac@ethz.ch}
}
\maketitle

\begin{abstract}
%
Information Cascades Model captures dynamical properties of user activity in a social network. In this work, we develop a novel framework for \emph{activity shaping} under the Continuous-Time Information Cascades Model which allows the administrator for local control actions by allocating targeted resources that can alter the spread of the process. 
Our framework employs the optimization of the \emph{spectral radius of the Hazard matrix}, a quantity that has been shown to drive the maximum influence in a network, while enjoying a simple convex relaxation when used to minimize the influence of the cascade. In addition, use-cases such as \emph{quarantine} and \emph{node immunization} are discussed to highlight the generality of the proposed activity shaping framework. 
Finally, we present the \emph{NetShape} influence minimization  method which is compared favorably to baseline and state-of-the-art approaches through simulations on real social networks. 
%
\end{abstract}

\section{Introduction}\label{sec:introduction}
%
The emergence of large scale social networks offers the opportunity to study extensively diffusion processes in various disciplines, including sociology, epidemiology, marketing, computer systems' security, \etc. Theoretical studies gave valuable insights on such processes by defining quantities tightly related with the systemic behavior (\eg epidemic threshold, extinction time) and describing how a diffusion unfolds from an initial set of contagious nodes. This quantification of systemic properties can on one hand help the assessment of certain economic/health/social risks, while on the other hand enable \emph{diffusion process engineering} that aims either to suppress or enhance the spreading.

Among the earliest works that drew a line between epidemic spreading and the structural properties of the underlying network is that in \citep{wang2003epidemic}. Under a mean field approximation of an SIR epidemic model on a graph, they found that the epidemic threshold is proportional to the \emph{spectral radius} of the adjacency matrix. 
Follow-up works verified this relation and broadened the discussion.   
In \citep{prakash2012threshold} the S$^*$I$^2$V$^*$ model was presented as a generalization of numerous virus propagation models (VPM) of the literature. It was also made possible to generalize the result of \citep{wang2003epidemic} to that generic VPM. Based on these works, several research studies have been presented on the epidemic control on networks, mainly focusing on developing \emph{immunization} strategies (elimination of nodes) and \emph{quarantine} strategies (elimination of edges). The eigenvalue perturbation theory was among the main analytical tools used, see for example \citep{tong2010vulnerability,van2011decreasing,tong2012gelling}.

The Information Cascade Model (ICM) \citep{chen2013information} is a modern family of models that considers heterogeneous node-to-node transmission probabilities. ICM fits well to problems related to information diffusion on social networks and, among others, finds straightforward applications in digital marketing \citep{kempe2003maximizing}. Indeed, ICMs were used to fit real information cascade data and observed `infection' times of nodes in the MemeTracker dataset \citep{leskovec2009meme}. In another work, the aim was to infer the edges of a diffusion network and estimate the transmission rates of each edge that best fits the observed data \citep{rodriguez2011uncovering}.

Similar theoretical results to those discussed above for VPMs have been given for ICM as well. Under discrete- or continuous-time ICM, it has been shown that the epidemic threshold depends on the \emph{spectral radius} of a matrix built upon the edge transmission probabilities, termed as \emph{Hazard matrix} \citep{scaman2015anytime,lemonnier2014tight}. 

On the algorithmic side, \citep{kempe2003maximizing} formulated for the first time the \emph{influence maximization} problem under the ICM. It was proved that it is an NP-hard problem and remains NP-hard to approximate it within a factor $1 \op{-} 1/e$. It was also proven that the influence is a sub-modular function of the set of initially contagious nodes (referred to as \emph{influencers}) and the authors proposed a greedy Monte-Carlo-based algorithm as an approximation. A number of subsequent studies were focused on improving that technique \citep{ohsaka2014fast,leskovec2007cost}. Notably, today's state-of-the-art techniques on influence control under the ICM are still based on Monte-Carlo simulations and a greedy mechanism to select the actions sequentially.

Besides influence maximization, various questions regarding how one could apply suppressive interventions have become a hot topic in recent years. For instance, the aim could be to reduce the spread of false and harmful information in a social network. Suppressive scenarios like the latter are also possible in the same modeling context; the optimization problem would be the minimization of the spread of a piece of malicious information in the network, \eg through the decrease in the probability for some users to share the false content to their contacts. 

In this paper we discuss the generic \emph{offline influence optimization}, or \emph{activity shaping} through local intervention actions that affect the spread. The purpose can be either to minimize the influence with suppressive actions, or to maximize it with enhancive actions. We seek for an efficient strategy to use the available budget of actions in order to serve better one of those opposing aims. Our approach is that we frame this as a generalized optimization problem under the ICM which has a convex continuous relaxation. We propose a class of algorithms based on the optimization of the spectral radius of the Hazard matrix using a projected subgradient method. For these algorithms, which can address both the maximization and the minimization problem, we provide theoretical analysis. We also investigate standard case-studies of the latter, such as the quarantine (\eg see \citep{tong2012gelling,van2011decreasing}) and the node immunization problem (see \citep{tong2010vulnerability}). The proposed algorithm called NetShape is easy to implement and compares favorably to standard baselines and state-of-the-art competitors in the reported experimental results.

\begin{table*}[ht]
\centering 
\footnotesize
\begin{tabular}{ l | l }
  \toprule
	\textbf{Symbol} & \textbf{Description}\\
	\midrule
	$\ind\{\text{<condition>}\}$ & indicator function\\
	$\one$ & vector with all values equal to one\\
	$\|X\|_\ell$ & $\ell$-norm for a given vector $X$: \eg $\|X\|_1 = \sum_{ij} X_{ij}$, or generally $\|X\|_\ell = (\sum_{ij} X_{ij}^\ell)^{1/\ell}$\\
	$M \elmult M'$ & the Hadamard product between matrices $M$ and $M'$ (\ie coordinate-wise multiplication)\\
	$\mu_{\pi(1)} \op{\geq} \mu_{\pi(2)} ...$ & ordered values of vector $\mu$ using the order-to-index bijective mapping $\pi$ \\
	\midrule	
	$\mathcal{G}, \mathcal{V}, n, \mathcal{E}, E$ & network $\mathcal{G}= \{\mathcal{V},\mathcal{E}\}$ of $n = |\mathcal{V}|$ nodes and $E = |\mathcal{E}|$ edges, where $\mathcal{V}$, $\mathcal{E}$ are the sets of nodes and edges\\
	$(i,j)$           & edge $(i,j) \in \mathcal{E}$ of the graph between nodes $i$ and $j$\\
	$A$							& network's adjacency matrix $A\in\{0,1\}^{n\times n}$\\
	$S_0, n_0$      & subset $S_0 \subset \mathcal{V}$ of $n_0 = |S_0|$ influencer nodes from which the IC initiates\\
	\midrule	
  $\FunMat$ 		& $n\times n$ \emph{Hazard matrix} $[\FunMat_{ij}]_{ij}$ of  non-negative integrable \emph{Hazard functions} defined over time\\
	
	$\FunSet$			& set of feasible Hazard matrices $\FunSet\subset\real_+\rightarrow\real_+^{n\times n}$, where $\FunMat$ is one of its elements\\
	$\DeltaF    $ & matrix of the integrated difference of two Hazard matrices in time: $\DeltaF = \int_0^{+\infty}(\hat{\FunMat}(t) - \FunMat(t))dt$\\
	$\tau_i$  		  & time $\tau_i\in\real_+\cup\{+\infty\}$ when the information reached node $i$ during the process\\
	$\sigma(S_0)$   & \emph{influence}: the number of contagious nodes after the diffusion started from the set $S_0$\\
	$\HazSpec(\FunMat)$ & the largest eigenvalue of the symmetrized and integrated Hazard matrix $\FunMat$\\
	\midrule
	$X$ & control actions matrix $X\in[0,1]^{n\times n}$ representing the amount of action taken on each edge\\
	$x$ & control actions vector with $x\in[0,1]^n$ representing the amount of action taken on each node\\
	$k$ & budget of control actions $\budget\in(0,E)$ or $\budget\in(0,n)$ for actions on edges and nodes, respectively \\
	\bottomrule
\end{tabular}
\caption{Index of main notations.}
\label{tab:notations} 
\end{table*}

\section{Diffusion model and influence bounds}\label{sec:preliminaries}
%
Let $\mathcal{G} = (\mathcal{V},\mathcal{E})$ be a directed graph of $n = |\mathcal{V}|$ nodes and $E = |\mathcal{E}|$ edges, and the adjacency matrix of $\mathcal{G}$ as $A\in\{0,1\}^{n\times n}$ \st $A_{ij} = 1 \Leftrightarrow (i,j)\in \mathcal{E}$. 
We denote as $S_0 \subset \mathcal{V}$ a set of $n_0 = |S_0|$ \emph{influencer nodes} that are initially contagious for a piece of information and can thus influence, or `infect', others.  The spread of information from the contagious nodes is modeled using the following continuous-time diffusion model in which each node $i$ can infect its neighbor $j$ independently according to a time-dependent transmission rate. 
Let $\tau_i\in\real_+\cup\{+\infty\}$ the time when the information reached node $i$ and made it contagious. Note that this quantity may be infinite if node $i$ did not receive at all the information during the process. The reader may find helpful the index of our basic notation in \Tab{tab:notations}.

\begin{definition} \emph{Hazard function} $\FunMat_{ij}(t)$ {--} For every edge $(i,j)\in \mathcal{E}$ of the graph, $\FunMat_{ij}$ is a non-negative integrable function that describes the time-dependent transmission rate from node $i$ to node $j$ after $i$'s infection.
\end{definition}

\begin{definition} \emph{Continuous-Time Information Cascade Model} $\CTIC$ {--} 
This is a stochastic diffusion process defined as follows: at time $s = 0$, only the influencer nodes of $S_0$ are infected. Then, each node $i$ that receives the contagion at time $\tau_i$ may transmit it at time $s \ge \tau_i$ along an outgoing edge $(i,j)\in \mathcal{E}$ with stochastic rate of occurrence $\FunMat_{ij}(s - \tau_i)$. We denote as 
$\FunMat = [\FunMat_{ij}]_{ij}$ the $n\times n$ \emph{Hazard matrix} containing as elements the individual Hazard functions and, respectively, $\FunMat(t) = [\FunMat_{ij}(t)]_{ij}$ the evaluation of all functions at \emph{relative time} $t$ after each infection time $\tau_i$. Essentially, network edges represent non-zero Hazard functions, and 
\begin{equation}\label{eq:nonzero-F_ij}
(i,j)\in \mathcal{E} \ \Leftrightarrow \ \exists t\geq 0\ \ \st \ \FunMat_{ij}(t) \neq 0.
\end{equation}
\end{definition}

\begin{definition} \emph{Influence} $\sigma(S_0)$ {--} In the $\CTIC$ model, the influence of a set of influencer nodes $S_0 \subset \mathcal{V}$ is defined as the number of infected nodes at the end of diffusion:
\begin{equation}\textstyle
\sigma(S_0) = \mathbbm{E}_{S_0}\left[\sum_{i\in V}\ind\{\tau_i < +\infty\}\right]\!,
\end{equation}
provided that the influencers are initially infected and contagious, thus, always $\sigma(S_0)\geq |S_0|$.
\end{definition}

In order to derive upper bounds for the influence under the $\CTICnoF$, we use the concept of \emph{Hazard radius} introduced in \citep{lemonnier2014tight} that is highly correlated to the influence. This is in analogy to the spectral radius of the adjacency matrix for virus propagation models \citep{tong2012gelling,FaloutsosTKDE2016}; recall that the spectral radius of a square matrix is defined as its largest eigenvalue. 

\begin{definition} \emph{Hazard radius} $\HazSpec(\FunMat)$ {--}
For a diffusion process $\CTIC$, $\HazSpec(\FunMat)$ is the largest eigenvalue of the symmetrized and integrated Hazard matrix:
\begin{equation}
\HazSpec(\FunMat) = \rho\left(\int_0^{+\infty} \!\! \frac{\FunMat(t) + \FunMat(t)^\top}{2} dt\right)\!,
\end{equation}
where $\rho(\cdot) = \max_i |\lambda_i|$, and $\lambda_i$ are the eigenvalues of the implied input matrix.
\end{definition}

Therefore, despite $\FunMat$ being a complex algebraic object, it is easy compute the spectral radius since that is computed on the integrated and symmetrized version of $\FunMat$.

The following proposition provides an upper bound for the influence of any set of influencers that depends on the Hazard radius, and is actually a simple corollary of Proposition 1 in \citep{lemonnier2014tight}.

\begin{proposition}\label{prop:bounds}
Let $S_0 \subset \mathcal{V}$ be a set of $n_0$ influencer nodes, and $\HazSpec(\FunMat)$ the Hazard radius of a $\CTIC$ information cascade. Then, the influence of $S_0$ in $\CTIC$ is upper bounded by:
\begin{equation}
\sigma(S_0) \le  n_0 + \gamma(n-n_0),
\end{equation}
where $\gamma \in [0,1]$ is the unique solution of the equation:
\begin{equation}
\gamma - 1 + \exp\left(-\HazSpec(\FunMat) \gamma - \frac{\HazSpec(\FunMat) n_0}{\gamma(n-n_0)}\right) = 0.
\end{equation}
\end{proposition}
\begin{proof}
This result immediately follows from Proposition $1$ of \citep{lemonnier2014tight} and the fact that, using their notations, $\HazMat(S_0)_{ij} = \ind\{j\in S_0\}\cdot\int_0^{+\infty} \!\! \FunMat_{ij}(t)dt\leq\int_0^{+\infty} \!\! \FunMat_{ij}(t)dt$. Then, using the Perron-Frobenius theorem leads to the inequality $\rho(\frac{\HazMat(S_0) + \HazMat(S_0)^\top}{2})\leq \HazSpec(\FunMat)$. 
\end{proof}

In essence, this result implies that the maximum influence cannot exceed a proportion $\gamma$ of the network that is non-decreasing with $\HazSpec(\FunMat)$, and displays a sharp transition between a sub-critical and super-critical regime. 
Also, previous experimental analysis \citep{lemonnier2014tight} showed that this upper bound is sharp for a large class of networks.

\inlinetitle{Our main line of contribution}{:} What we put forward in this work is directly derived by the discussion above and is the idea that \emph{the Hazard radius $\HazSpec$ can be used as a proxy to minimize or maximize the maximum influence under the IC model in a given social network}, and this way perform activity shaping. This is the main line of our contribution and the motivation for the \methodName algorithm presented in \Sec{sec:solution_and_proposed_algorithm}.

\section{Monitoring Information Cascades}\label{sec:setup_of_the_problem}
%
The aim of this work is to provide an efficient approach to the generic problem of optimizing influence (maximizing or minimizing) using actions that can shape, \ie modify, the activity of single users. For instance, a marketing campaign may have a certain advertisement budget that can be used on targeted users of a social network. While these targeted advertisements are usually represented as new influencer nodes that will spread the piece of information, we rather consider the more refined and general case in which each targeted advertisement will essentially alter the Hazard functions $\FunMat_{ij}$ associated to a target node $i$, thus increasing, or decreasing, the probability for $i$ to propagate by sharing the information with its neighbors. 

Our generic framework assumes that a \emph{set of feasible Hazard matrices} $\FunSet\subset\real_+\rightarrow\real_+^{n\times n}$ is available to the marketing agency. This set virtually contains all admissible policies that one could apply to the network. Then, the concern of the agency is to find the Hazard matrix $\FunMat\in\FunSet$ that minimizes, or maximizes depending on the task of interest, the influence. Particular instances of this generic framework are presented in \Sec{sec:case_studies}.

\begin{problem} \emph{Determining the optimal feasible policy} -- \label{probl:policy_influence}
Given a graph $\mathcal{G}$, a number of influencers $n_0$ and a set of admissible policies $\FunSet$, find the optimal policy:
\begin{equation}\label{eq:F*}
\FunMat^* =  \argmin_{ \FunMat \in\FunSet } ~\sigma_{n_0}^*(\FunMat),
\end{equation}
where $\sigma_{n_0}^*(\FunMat) = \max\{\sigma(S_0) : S_0\subset \mathcal{V} \mbox{ and } |S_0| = n_0\}$ is the optimal influence (according to \Eq{eq:F*} this is the minimum) over any possible set of $n_0$ influencer nodes.
\end{problem}

\Problem{probl:policy_influence} cannot be solved exactly in polynomial time. The exact computation of the maximum influence $\sigma_{n_0}^*(\FunMat)$ is already a hard problem on its own, and minimizing this quantity adds an additional layer of complexity due to the non-convexity of the maximum influence \wrt the Hazard matrix (note: $\FunMat\mapsto\sigma_{n_0}^*(\FunMat)$ is positive, upper bounded by $n$ and not constant).

\begin{proposition}\label{prop:hardness}
For any size of the set of influencers $n_0$, the computation of $\sigma^*_{n_0}(\FunMat)$ is $\#$P-hard.
\end{proposition}
\begin{proof}
We prove the theorem by reduction from a known $\#$P-hard function: the computation of the influence $\sigma(S_0)$ given a set of influencers $S_0$ of size $n_0$ (see Theorem 1 of \citep{wang2012scalable}). Indeed, let $\mathcal{CTIC}(\FunMat)$ be an Independent Cascade model defined on $\mathcal{G} = (\mathcal{V},\mathcal{E})$. We can construct a new graph $\mathcal{G}' = (\mathcal{V}',\mathcal{E}')$ as follows: for each influencer node $i\in S_0$, add a directed chain of $n$ nodes $\{v_{i,1},\dots,v_{i,n}\}\subset\mathcal{V}'$ and connect $v_{i,n}$ to $i$ by letting the transmission probabilities along the edges be all equal to one. Then, the maximum influence $\sigma^*_{n_0}$ is achieved with the nodes $S_0' = \{v_{i,1}\,:\,i\in S_0\}$ as influencer, and $\sigma^*_{n_0} = n\,n_0 + \sigma(S_0)$. The result follows from the $\#$P-hardness of computing $\sigma(S_0)$ given $S_0$.
\end{proof}

The standard way to approximate the maximum influence is to employ incremental methods where the quality of each potential influencer is assessed using a Monte-Carlo approach. 
In the following, we assume that the feasible set $\FunSet$ is convex and included in a ball of radius $R$. 
Also, the requirement of \Eq{eq:nonzero-F_ij}, that network edges correspond to non-zero Hazard functions, holds for every feasible policy $\FunMat\in\FunSet$. Therefore, the number of edges $E$ upper bounds the number of non-zero Hazard functions for any $\FunMat\in\FunSet$.

\begin{remark}
Although \Problem{probl:policy_influence} focuses on the minimization of the maximum influence, the algorithm presented in this paper is also applicable to the opposite task of influence maximization.
Having a common ground for solving these opposite problems can be particularly useful for applications where both opposing aims can interest different actors, for instance in market competition. For the maximization, our algorithm would use a gradient ascent instead of a gradient descent optimization scheme. While the performance of the algorithm in that case may be competitive to state-of-the-art influence maximization algorithms, the nonconvexity of this problem prevents us from providing any theoretical guarantees regarding the quality of the final solution.
\end{remark}

\section{\methodName: an algorithm for monitoring Information Cascades}\label{sec:solution_and_proposed_algorithm}
%
Bearing in mind the computational intractability of solving exactly the influence optimization problem, we propose to exploit the upper bound given in \Proposition{prop:bounds} as a heuristic for approximating the maximum influence. This approach can be seen as a \emph{convex relaxation} of the original NP-Hard problem, and allows the use of convex optimization algorithms for this particular problem. The relaxed optimization problem thus becomes:
\begin{equation}\label{eq:minrho}
\FunMat^* =  \argmin_{ \FunMat \in\FunSet } ~\HazSpec(\FunMat).
\end{equation}

When the feasible set $\FunSet$ is convex, this optimization problem is also convex and our proposed method called \emph{\methodName} uses a simple \emph{projected subgradient descent} (see e.g. \citep{bubeck2015convex}) in order to find its minimum and make sure that the solution lays in $\FunSet$. However, special care should be taken to perform the gradient step since, although the objective function $\HazSpec(\FunMat)$ admits a derivative \wrt the norm
\begin{equation}\label{eq:funMatNorm}\textstyle
\|\FunMat\| = \sqrt{\sum_{i,j} \left(\int_0^{+\infty}\!\!|\FunMat_{ij}(t)|dt\right)^2},
\end{equation}
the space of matrix functions equipped with this norm is only a Banach space in the sense that the norm $\|\FunMat\|$ cannot be derived from a well chosen scalar product. Since gradients only exist in Hilbert spaces, gradient-based optimization methods are not directly applicable. 

In the \methodName algorithm, the gradient and projection steps are performed on the \emph{integral} of the Hazard functions $\int_0^{+\infty}\!\FunMat_{ij}(t)dt$ by solving the optimization problem bellow:
\begin{equation}\label{eq:projStep}
\FunMat^* \op{=} \argmin_{\hat{\FunMat}\in\FunSet}  \left\|\int_0^{+\infty} \!\!\! \left(\!\!\hat{\FunMat}(t) - \FunMat(t)\right)dt + \eta \, u_\FunMat u_\FunMat^\top\right\|_2\!,
\end{equation}
where $\eta > 0$ is a positive gradient step, $u_\FunMat$ is the eigenvector associated to the largest eigenvalue of the matrix $\int_0^{+\infty} \frac{\FunMat(t) + \FunMat(t)^\top}{2}dt$, and $u_\FunMat u_\FunMat^\top$ is a subgradient of the objective function, as provided by the following proposition.

\begin{proposition}
A subgradient of the objective function $f(M) = \rho\big(\frac{M+M^\top}{2}\big)$ in the space of integrated Hazard functions, where $M$ is a matrix, is given by the matrix:
\begin{equation}
\nabla f(M) = u_M u_M^\top,
\end{equation}
where $u_M$ is the eigenvector associated to the largest eigenvalue of the matrix $\frac{M+M^\top}{2}$.
\end{proposition}
\begin{proof}
For any matrix $M$, let $f(M) = \rho\big(\frac{M+M^\top}{2}\big) = \max_{x~:~\|x\|_2=1} x^\top M x$, and $u_M$ be such an optimal vector. Then, we have $f(M\op{+}\varepsilon) = u_{M\op{+}\varepsilon}^\top (M\op{+}\varepsilon) u_{M+\varepsilon} \geq u_M^\top (M+\varepsilon) u_M = f(M) + u_M^\top \varepsilon u_M$, and, since $u_M^\top\,\varepsilon\,u_M = \left\langle u_M u_M^\top, \varepsilon \right\rangle$, $u_M u_M^\top$ is indeed a subgradient for $f(M)$.
\end{proof}

\begin{algorithm}[t]
\caption{\textbf{--} \methodName meta-algorithm}
\label{alg:generic}
\begin{algorithmic}[1]
    \REQUIRE{feasible set $\FunSet\subset\real_+\rightarrow\real_+^{n\times n}$, radius $R>0$ of $\FunSet$, initial Hazard matrix $\FunMat\in\FunSet$, approx. parameter $\epsilon>0$}
    \ENSURE{Hazard matrix  $\FunMat^*\in\FunSet$}
   \vspace{1mm}
   \STATE $\FunMat^* \leftarrow \FunMat$
   \STATE $T \leftarrow \lceil\frac{R^2}{\epsilon^2}\rceil$
   \FOR{$i = 1$ to $T-1$}
			\STATE $u_\FunMat \op{\leftarrow}$\!\!
			the eigenvector assoc. to spectral radius $\HazSpec(\FunMat)$
			\STATE $\eta \leftarrow \frac{R}{\sqrt{i}}$
			\STATE $\FunMat \leftarrow \argmin_{\hat{\FunMat}\in\FunSet} \left\|\int_0^{+\infty} \!\! \left(\hat{\FunMat}(t) - \FunMat(t)\right)dt + \eta \, u_\FunMat u_\FunMat^\top\right\|_2$
   		\STATE $\FunMat^* \leftarrow \FunMat^* + \FunMat$
   \ENDFOR
	 \RETURN $\frac{1}{T}\FunMat^*$
\end{algorithmic}
\end{algorithm}

The projection step of line~6 in \Alg{alg:generic} is an optimization problem on its own, and \methodName algorithm is practical if and only if this optimization problem is simple enough to be solved. In the next sections we will see that, in many cases, this optimization problem can be solved in near linear time \wrt the number of edges of the network (\ie $\mathcal{O}(E\ln E)$), and is equivalent to a projection on a simplex.

\subsection{Convergence and scalability}
%
Due to the convexity of the optimization problem in \Eq{eq:minrho}, \methodName finds the global miniminum of the objective function and, as such, may be a good candidate to solve \Problem{probl:policy_influence}.
The complexity of the \methodName algorithm depends on the complexity of the projection step in \Eq{eq:projStep}. Each step of the gradient descent requires the computation of the first eigenvector of an $n\times n$ matrix, which can be computed in $\mathcal{O}(E\ln{E})$, where $E$ is the number of edges of the underlying graph. In most real applications, the underlying graph on which the information is diffusing is \emph{sparse}, in the sense that its number of edges $E$ is small compared to $n^2$. 
\begin{proposition}
Assume that $\FunSet$ is a convex set of Hazard matrices included in a ball of radius $R>0$ \wrt the norm in \Eq{eq:funMatNorm}, and that the projection step in \Eq{eq:projStep} has complexity at most $\mathcal{O}(E\ln{E})$. Then, the \methodName algorithm described in \Alg{alg:generic} converges to the minimum of \Eq{eq:minrho}. Moreover, the complexity of the algorithm is $\mathcal{O}(\frac{R^2}{\epsilon^2}E\ln{E})$.
\end{proposition}
\begin{proof}
This is a direct application of the projected subgradient descent to the problem:
\begin{equation}
\HazMat^* = \argmin_{\HazMat\in\mathbbm{H}} \rho\left(\frac{\HazMat + \HazMat^\top}{2}\right),
\end{equation}
where $\mathbbm{H} = \{\int_0^{+\infty} \!\! \FunMat(t)dt\in\mathbbm{R}^{n\times n}~:~\FunMat\in\FunSet\}$ is the set of feasible Hazard matrices. The convergence rate of such an algorithm can be found in \citep{bubeck2015convex}.
\end{proof}

\begin{remark} The corresponding maximization problem is not convex anymore and only convergence to a local maximum can be expected. However, when the changes in the Hazard functions are relatively small (\eg inefficient control actions, or only a limited number of treatments available to distribute), then \methodName achieves fairly good performance.
\end{remark}

\begin{algorithm}[t]
\caption{\textbf {--} \methodName partial quarantine problem}
\label{alg:pqp}
\begin{algorithmic}[1]
    \REQUIRE{graph $\mathcal{G} = (\mathcal{V},\mathcal{E})$, matrices of Hazard functions \emph{before} and \emph{after} treatment $\FunMat,\hat{\FunMat}\in\FunSet$, approximation parameter $\epsilon>0$, number of treatments $\budget$}
    \ENSURE{matrix of Hazard functions $\FunMat^*\in\FunSet$}
		\vspace{1mm}
	 \STATE $X \leftarrow 0$, $X^* \leftarrow 0$
	 \STATE $F \leftarrow \int_0^{+\infty}\FunMat(t)dt$
	 \STATE $\DeltaF \leftarrow \int_0^{+\infty}(\hat{\FunMat}(t)dt - \FunMat(t))dt$
   \STATE $R \leftarrow \sqrt{\budget}\max_{ij}\DeltaF_{ij}$
	 \STATE $T \leftarrow \lceil\frac{R^2}{\epsilon^2}\rceil$
   \FOR{$i = 1$ to $T-1$}
			\STATE $M \leftarrow F + X\elmult\DeltaF$
			\STATE $u \leftarrow$ the largest eigenvector of $\frac{1}{2}(M + M^\top)$
			\STATE $Y \leftarrow X\elmult\DeltaF - \frac{R}{\sqrt{i}} u u^\top$\hfill \texttt{/\!/\,\rotatebox[origin=c]{90}{\raisebox{-1mm}{$\Lsh$}}\,projection step (\Alg{alg:projStep})}
			\STATE $X \leftarrow \argmin_{X'\in[0,1]^{n\times n}, \|X'\|_1 \leq \budget} \left\|X' \elmult \DeltaF - Y\right\|_2$ 
   		\STATE $X^* \leftarrow X^* + X$
   \ENDFOR
	 \RETURN $\FunMat^* = (1-\frac{1}{T}X^*)\elmult\FunMat + \frac{1}{T}X^*\elmult\hat{\FunMat}$
\end{algorithmic}
\end{algorithm}
\begin{algorithm}[t]
\caption{\textbf{--} Projection step for partial quarantine (\Alg{alg:pqp})\!\!}
\label{alg:projStep}
\begin{algorithmic}[1]
    \REQUIRE{
		$\delta, y \in \real^E$, budget $\budget\in(0,E)$}
    \ENSURE{control actions vector $x'$}
		\vspace{1mm}
	 \FOR{$i=1$ to $E$}
	   \STATE $\mu_{i} \leftarrow 2\delta_i y_i$
	   \STATE $\mu_{E+i} \leftarrow 2\delta_i (y_i - \delta_i)$
   \ENDFOR
	 \STATE sort $\mu$ into $\mu_{\pi(1)}\geq\mu_{\pi(2)}\geq\dots\geq\mu_{\pi(2E)}$
	 \STATE $d \leftarrow 0$, $s \leftarrow 0$, $i \leftarrow 1$
	 \WHILE{$s < \budget$ and $\mu_{\pi(i)}\geq 0$}
		   \STATE $d \leftarrow d + \ind\{\pi(i) \leq E\}\frac{1}{2\delta_{\pi(i)}^2} - \ind\{\pi(i) > E\}\frac{1}{2\delta_{\sigma(i)-E}^2}$
	     \STATE $s \leftarrow s + d(\mu_{\pi(i)} - \mu_{\pi(i+1)})$
	     \STATE $i \leftarrow i+1$
	 \ENDWHILE
	 \STATE $z \leftarrow \max\{0, \mu_{\sigma(i)} + \frac{s - \budget}{d}\}$
	 \RETURN $x'$ \st $x'_i = \max\{0, \min\{\frac{2\delta_i y_i - z}{2\delta_i^2}, 1\}\}$
\end{algorithmic}
\end{algorithm}

\section{Case studies}\label{sec:case_studies}
%
In this section, we illustrate the generality of our framework  by reframing well-known influence optimization problems using \Problem{probl:policy_influence} and deriving the corresponding variants of the \methodName algorithm. For the sake of notations' simplicity, we denote as $M \elmult M'$ the Hadamard product between the two matrices (\ie coordinate-wise multiplication), as 
$\DeltaF = \int_0^{+\infty} \! (\hat{\FunMat}(t) - \FunMat(t))dt$ the matrix with the integrated coordinate-wise difference of two Hazard matrices in time, and as $\one\in\real^n$ the all-one vector (see notations in \Tab{tab:notations}).

\subsection{Partial quarantine}
%
The \emph{quarantine} problem addresses the removal of a small number of edges in order to minimize the spread of the contagion. Such an approach is highly interventional, in the sense that it totally removes edges, but in order to be practical it has to remain at low scale and affect a small amount of edges. This is the reason why it is mostly appropriate for dealing with the initial very few infections.

The \emph{partial quarantine} setting is a relaxation where one is interested to decrease the transmission probability along a set of targeted edges by using local and expensive actions.

\begin{definition}\emph{Partial quarantine}~--~\label{def:quarantine_IC}
Consider that a marketing campaign has $\budget$ control actions to distribute in a network $\mathcal{G} = (\mathcal{V},\mathcal{E})$. For each edge $(i,j)\in \mathcal{E}$, let $\FunMat_{ij}$ and $\hat{\FunMat}_{ij}$ be the Hazard matrices \emph{before} and \emph{after} applying control actions, respectively. If $X\in[0,1]^{n\times n}$ is the control actions matrix and $X_{ij}$ represents the amount of suppressive action taken on edge $(i,j)$, then the set of feasible policies is:
\begin{equation}
\!\FunSet \op{=} \left\{\!(1\!\op{-}\!X)\op{\elmult}\FunMat + X\op{\elmult}\hat{\FunMat}\,:\,X\!\op{\in}[0,1]^{n\times n}\!\!,\,\|X\|_1\!\op{\leq} \budget\!\right\}\!\!.\!\!\!\!\!
\end{equation}
\end{definition}

\noindent\emph{Example}:~For a non-negative scalar $\epsilon \geq 0$, we may consider $\hat{\FunMat} = (1\op{-}\epsilon)\FunMat$ in order to model the suppression of selected transmission rates; formally: $\FunSet = \left\{(1-\epsilon X)\elmult\FunMat : X\in[0,1]^{n\times n}, \|X\|_1 \leq \budget\right\}$. Importantly, for the special case where $\epsilon = 1$, this problem becomes equivalent to the setting discussed in \citep{tong2012gelling} and \citep{van2011decreasing}.

A straightforward adaptation of \Alg{alg:generic} to this setting leads to the \methodName algorithm for partial quarantine described in \Alg{alg:pqp}. The projection step is performed by \Alg{alg:projStep} on the flattened versions $x', \delta, y\in\real^E$ of the matrices $X'$, $\Delta$ and $Y$, and the parameter $R$ is chosen to upper bound $\max_{\FunMat'\in\FunSet}\|\FunMat' - \FunMat\|_2 = \max_{X\in[0,1]^{n\times n},\|X\|_1\leq k}\|X\elmult\Delta\|_2$.

\begin{lemma}
The projection step of \Alg{alg:generic} for the partial quarantine setting of \Definition{def:quarantine_IC} is:
\begin{equation}\label{eq:lemeq}
X^* = {\arg\min}_{x'\in[0,1]^E,~\|x'\|_1 \leq \budget} \left\|x' \elmult \delta - y\right\|_2,
\end{equation}
where $\delta$ and $y$ are flattened version of, respectively, $\DeltaF$ and $Y = X\elmult\DeltaF - \eta u_\FunMat u_\FunMat^\top$. Moreover, this problem can be solved in time $\mathcal{O}(E\ln{E})$ with \Alg{alg:projStep}, where $E$ is the number of edges of the network.
\end{lemma}
\begin{proof}
\Eq{eq:lemeq} directly follows from \Eq{eq:projStep} and the definition of $\FunSet$. \Alg{alg:projStep} is an extended version of the $L_1$-ball projection algorithm of \citep{DuchiEfficient2008}. KKT conditions for the optimization problem of \Eq{eq:lemeq} imply that $\exists z>0$ \st $\forall i$, $x_i' = \max\{0, \min\{\frac{2\delta_i y_i - z}{2\delta_i^2}, 1\}\}$. The algorithm is a simple linear search for this value.
Finally, the sorting step (\Alg{alg:projStep}, line 5) has the highest complexity of $\mathcal{O}(E\ln{E})$, and the loops perform at most $2E$ iterations, hence an overall complexity of $\mathcal{O}(E\ln{E})$.
\end{proof}

\subsection{Partial node immunization}\label{sec:node_partial_IC}
%
More often, control actions can only be performed on the nodes rather than the edges of a network, for example when considering targeted advertisements. In such a case, the effect must be aggregated over nodes in the following way.

\begin{definition}\emph{Partial node immunization}~--~\label{def:node_partial_IC} Consider that a marketing campaign has $\budget$ control actions to distribute in a network $\mathcal{G} = (\mathcal{V},\mathcal{E})$. For each edge $(i,j)\in \mathcal{E}$, let $\FunMat_{ij}$ and $\hat{\FunMat}_{ij}$ be the Hazard matrices \emph{before} and \emph{after} applying control actions, respectively. If $x\in[0,1]^n$ is the control actions vector and $x_i$ represents the amount of suppressive action taken on node $i$, then the set of feasible policies can be expressed as:
\begin{equation}
\!\,\FunSet \op{=} \left\{(1\!\op{-}\!x\one^\top)\op{\elmult}\FunMat + x\one^\top\op{\elmult}\hat{\FunMat}\,:\,x\!\op{\in}\![0,1]^n\!\!,\,\|x\|_1\!\op{\leq}\!\budget\right\}\!\!.\!\!\!
\end{equation}
\end{definition}

This setting corresponds to partial quarantine in which all outgoing edges of a node are impacted by a single control action. When $\hat{\FunMat} = 0$, this problem corresponds to the node removal problem (or vaccination), that consists in removing $\budget$ nodes from the graph in advance in order to minimize the contagion when that will appear (see \citep{tong2010vulnerability}).

Given a vector $x$, the projection problem to solve is:
\begin{equation*} \label{eq:projection2}
\begin{split}
x^* &= \argmin_{x'\in[0,1]^n, \|x'\|_1 \leq \budget} \left\| (x'\one^\top) \elmult \DeltaF - Y\right\|_2\\
&= \argmin_{x'\in[0,1]^n, \|x'\|_1 \leq \budget} \sum_{i} x_i^2 \bigg(\!\sum_j \DeltaF_{ij}^2\!\bigg) - 2x_i \bigg(\!\sum_j \DeltaF_{ij} Y_{ij}\!\bigg)\\
&= \argmin_{x'\in[0,1]^n, \|x'\|_1 \leq \budget} \left\| x' \elmult \delta' - y'\right\|_2,\hspace{2.2cm}\myEqCounter
\end{split}
\end{equation*}
where $\delta'_i = \sqrt{\sum_j \DeltaF_{ij}^2}$ and $y'_i = \frac{\sum_{j}\DeltaF_{ij} Y_{ij}}{\sqrt{\sum_j \DeltaF_{ij}^2}}$. Hence we can apply the projection step of \Alg{alg:projStep} for the partial node immunization problem using $\delta'$ and $y'$, and its complexity is $\mathcal{O}(n\ln{n})$.

\begin{table}[!b]
\beforecaptvskip
\centering
\footnotesize
\begin{tabular}{l|r|r|r}
	\toprule
	\textbf{Network} & \textbf{Nodes} & \textbf{Edges} & \textbf{Nodes in largest SCC} \\
	\midrule
	\dset{Facebook} &  $4,039$ &  $88,234$ &  $4,039$\,\,::    $100.0$\% \\
	\dset{Gnutella} & $62,586$ & $147,892$ & $14,149$\,\,:: \ \,$22.6$\% \\
	\dset{Epinions} & $75,879$ & $508,837$ & $32,223$\,\,:: \ \,$42.5$\%  \\
	\dset{Airports}  &    $332$ &   $2,126$ &    $332$\,\,::    $100.0$\% \\
	\bottomrule
\end{tabular}
\caption{Details of the benchmark real networks. The last column is the size of the strongly connected component.}
\label{tab:networks}
\end{table}

\begin{remark}
Since the upper bound of \Proposition{prop:bounds} holds as well for SIR epidemics \citep{kermack1932contributions} (see also \citep{scaman2015anytime}), this setting may also be used to reduce the spread of a disease using, for example, medical treatments or vaccines. More specifically, the Hazard matrix for an SIR epidemic is:
\begin{equation}\textstyle
\mathcal{H} = \ln\left(1+\frac{\beta}{\delta}\right) A,
\end{equation}
where $\delta$ is the recovery (or removal) rate and $\beta$ is the transmission rate along edges of the network, and $A$ the adjacency matrix. Then, a medical treatment may increase the recovery rate $\delta$ for targeted nodes, thus decreasing all Hazard functions on its outgoing edges, and the partial node immunization setting is applicable.
\end{remark}

\newcommand{\basicsubfigsize}{0.215}
\newcommand{\hsepar}{0.7pt}
\begin{figure*}[!t] \footnotesize
	\newcommand{\Xcut}{19.5}
	\hspace{-3.1pt}
	\subfigure{
		\begin{picture}(0,50)
			\put(0,50){\large{$\rho$}}
		\end{picture}
	}%
	\hspace{1pt}%
	\subfigure{
		\FPeval{\subfigsize}{(1+(\Xcut-0)/197)*\basicsubfigsize}
		\includegraphics[width=\subfigsize\linewidth, viewport=17 186 198 327, clip=true]{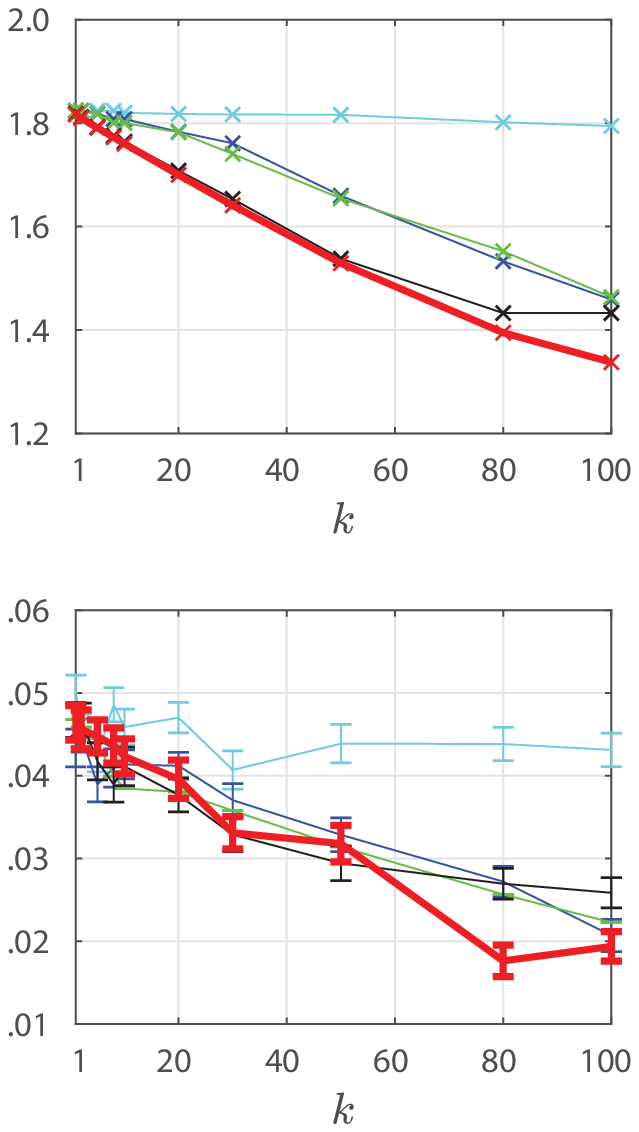}%
	}%
	\hspace{\hsepar}%
	\subfigure{
		\FPeval{\subfigsize}{(1+(\Xcut-0)/197)*\basicsubfigsize}
		\includegraphics[width=\subfigsize\linewidth, viewport=16 186 198 327, clip=true]{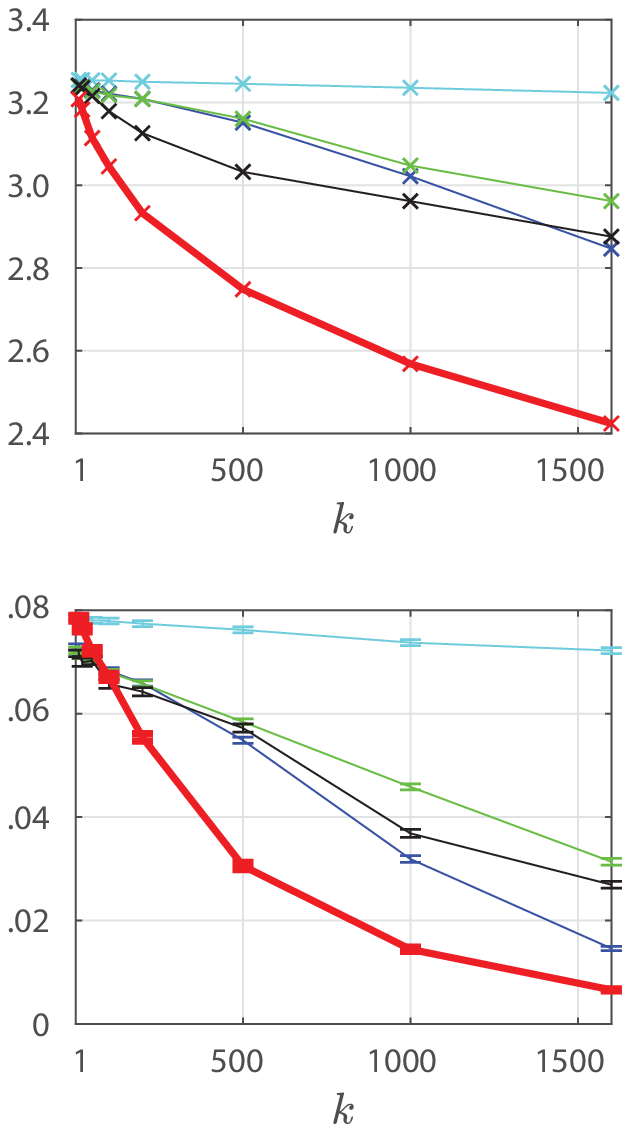}%
	}%
	\hspace{\hsepar}%
	\subfigure{
		\FPeval{\subfigsize}{(1+(\Xcut-0)/197)*\basicsubfigsize}
		\includegraphics[width=\subfigsize\linewidth, viewport=16 186 198 327, clip=true]{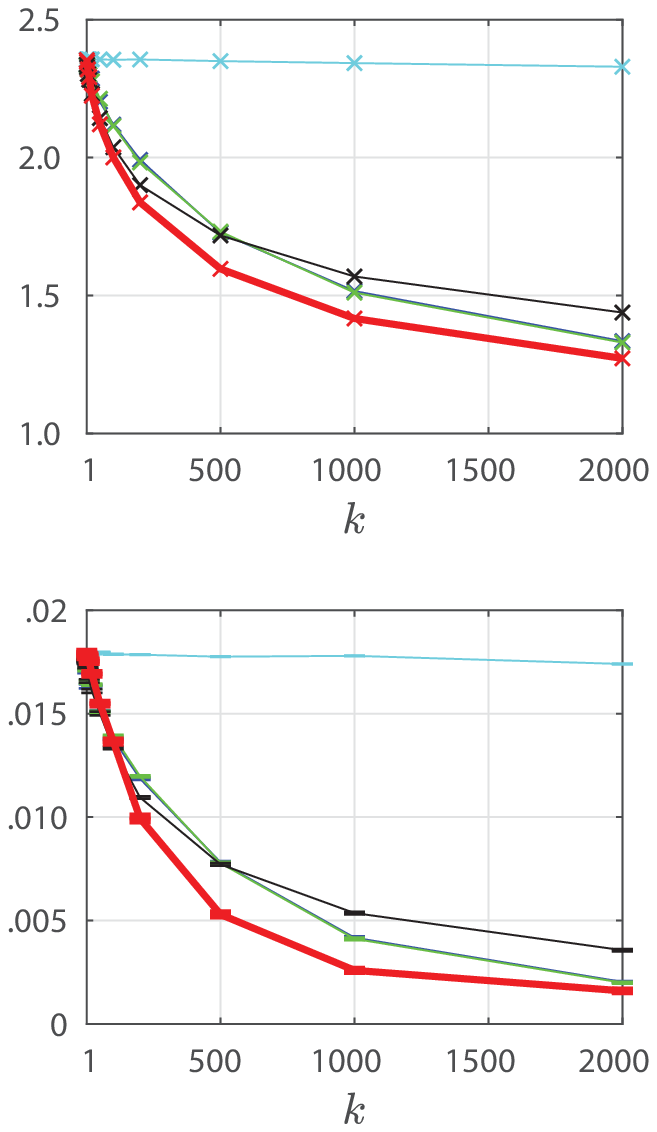}%
	}%
	\hspace{\hsepar}%
	\subfigure{
		\FPeval{\subfigsize}{(1+(\Xcut-0)/197)*\basicsubfigsize}
		\includegraphics[width=\subfigsize\linewidth, viewport=90 186 271 327, clip=true]{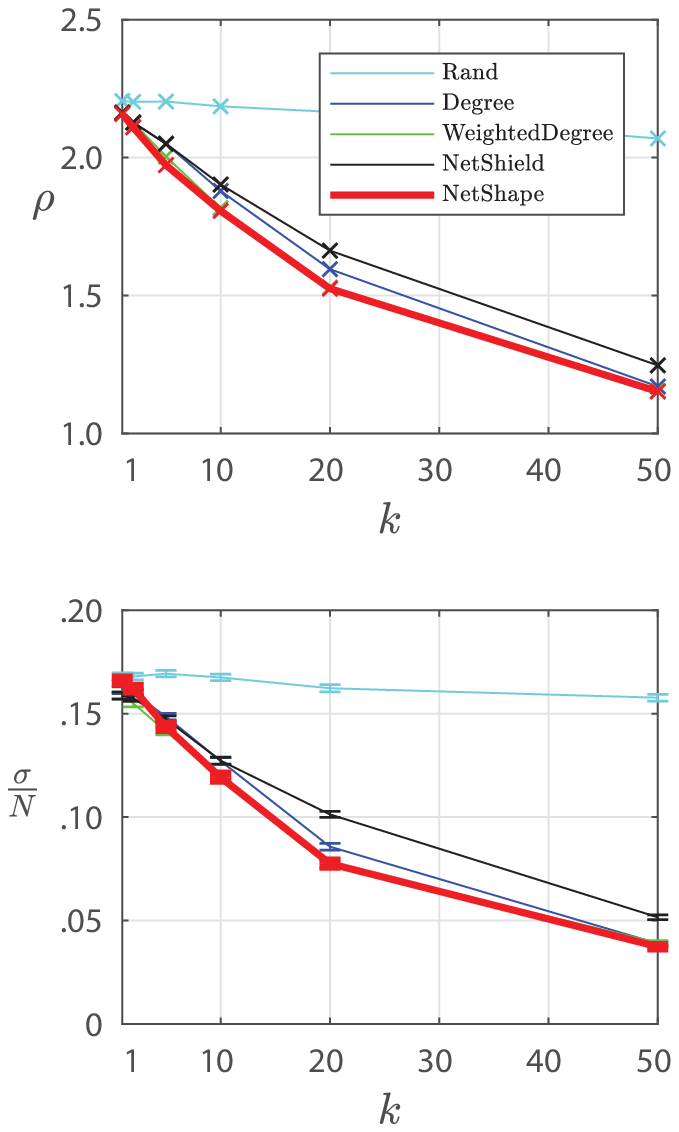}%
	}%
\vspace{-6pt}%
\addtocounter{subfigure}{-6}
\\%
%
	\renewcommand{\Xcut}{19.5}
	\subfigure{
		\begin{picture}(0,50)
			\put(-2,59){\Large{$\frac{\sigma}{n}$}}
		\end{picture}
	}%
  \subfigure[\dset{Facebook}]{\label{fig:res_facebook}
		\FPeval{\subfigsize}{(1+(\Xcut+5)/197)*\basicsubfigsize}
		\includegraphics[width=\subfigsize\linewidth, viewport=17 0.3 202 163, clip=true]{facebook_finetuned.eps}%
	}%
	\hspace{-2pt}%
	\subfigure[\dset{Gnutella}]{\label{fig:res_gnutella}
		\FPeval{\subfigsize}{(1+(\Xcut+5)/197)*\basicsubfigsize}
		\includegraphics[width=\subfigsize\linewidth, viewport=16 0.3 201 163, clip=true]{gnutella_finetuned.eps}%
	}%
	\hspace{-4pt}%
	\subfigure[\dset{Epinions}]{\label{fig:res_epinions}
		\FPeval{\subfigsize}{(1+(\Xcut+4)/197)*\basicsubfigsize}
		\includegraphics[width=\subfigsize\linewidth, viewport=13 -0.4 198 163, clip=true]{epinions_finetuned.eps}%
	}%
	\hspace{0.5pt}%
	\subfigure[\dset{Airports}]{\label{fig:res_airport}
		\FPeval{\subfigsize}{(1+(\Xcut-0)/197)*\basicsubfigsize}
		\includegraphics[width=\subfigsize\linewidth, viewport=90 0 270.9 163, clip=true]{airport_finetuned.eps}%
	}%
	\beforecaptvskip
	\caption{\footnotesize The effectiveness of the compared policies on benchmark real networks evaluated by two evaluation measures. For each network, at the top row is plotted the $\HazSpec(\FunMat)$ \vs budget $k$, and at the bottom row the expected proportion of infected nodes $\frac{\sigma}{n}$ \vs $k$. (a)~\dset{Facebook} network, by generating infection rates $p \op{\in} \{.0001, .001, .01\}$; (b)~\dset{Gnutella} network with $p \op{\in} \{.1, .3, .6\}$; (c)~\dset{Epinions} network with $p \op{\in} \{.005, .005, .05\}$; (d)~\dset{Airports} network with the original graph weights. Lower values are better.}\label{fig:real_data}
\end{figure*}

\section{Experiments}\label{sec:empirical_evaluation}
%
\inlinetitle{Setup and evaluation process}{.} In this section, we provide empirical evidence in support of our analysis and the performance of the proposed \emph{\methodName} algorithm. More specifically, we evaluate in the partial node immunization problem under ICM, as described in \Sec{sec:node_partial_IC}, and we provide comparative experimental results against several strategies, namely:

\noindent \,\,\,\,\,\,\textbf{i)}~\emph{Rand}: random selection of nodes; 

\noindent \,\,\,\,\,\textbf{ii)}~\emph{Degree}: selection of $k$ nodes with highest out-degree; 

\noindent \,\,\,\,\textbf{iii)}~\emph{Weighted-degree}: selection of $k$ nodes with highest sum of outgoing edge weight $w_{ij} = \int_0^{+\infty} \FunMat_{ij}(t)dt$. This strategy can also be seen as the optimization of the first influence lower bound $LB_1$ of \citep{NIPS2016_6347}. 

\noindent \,\,\,\,\textbf{iv)}~\emph{NetShield} algorithm \citep{tong2010vulnerability}.
Given the adjacency matrix of a graph, this outputs the best $k$-nodes to totally immunize so as to decrease the vulnerability of the graph. This is done by assigning to each node a \emph{shield-value} that is high for nodes with high eigenscore and no edges connecting them.
Note that, despite the fact that \emph{NetShield} is tailored for immunization on unweighted graphs, it is not general enough to account for weighted edges and partial immunization as in our experimental setting.

The evaluation is performed on four benchmark datasets (see \Tab{tab:networks}) and the results on each of them are presented in subfigures of \Fig{fig:real_data}: 
(a)~a network of `friends lists' from \dset{Facebook} \citep{snapnets}; 
(b)~the \dset{Gnutella} peer-to-peer file sharing network \citep{snapnets}, 
(c)~the who-trust-whom online review site \dset{Epinions.com}; 
(d)~a real \dset{Airports} network \citep{airport_dataset} with the weighted graph of flights that took place in 1997 connecting US airports; 

Note that for the first three networks, only an unweighted adjacency matrix is provided. The matrix of edge-transmission probabilities $\{p_{ij}\}$ is generated by a \emph{trivalency model}, which is to pick the $p_{ij}$ values uniformly at random from a small set of constants; in our case that is $\{p_{\text{low}}, p_{\text{med}}, p_{\text{high}}\}$ and the specific used values are mentioned explicitly for each dataset. 

In our experiments we evaluate the efficiency of the immunization policies along two measures for both of which lower values are better:

\noindent \,\,\,\,-~\emph{Spectral radius decrease}. We examine the extend of the decrease of the spectral radius of the Hazard matrix $\mathcal{F}$ and, hence, the decrease of the bound of the max-influence as described in \Proposition{prop:bounds}.

\noindent \,\,\,\,-~\emph{Expected influence decrease}. We compare the performance of policies in terms of \Problem{probl:policy_influence}. To this end, for each Hazard matrix $\mathcal{F}$, the influence is computed as the \emph{average number of infected nodes} at the end of over 1,000 runs of the information cascade $\mathcal{CTIC}$ while applying that specific Hazard matrix $\mathcal{F}$. Each time a single initial influencer is selected by the influence maximization algorithm Pruned Monte-Carlo \citep{ohsaka2014fast} by generating 1,000 vertex-weighted directed acyclic graphs (DAGs). 

In our empirical study, we focus on the scenario where the spectral radius of the original network is approximately one, which is the setting in which decreasing the spectral radius has the most impact on the upper bounds in \Proposition{prop:bounds} and \citep{lemonnier2014tight}). We believe that this intermediate regime is the most meaningful and interesting in order to test the different algorithms.

\inlinetitle{Results}{.} 
The results on each of the four real network datasets are shown in subfigures of \Fig{fig:real_data}. For each network, two vertically stacked plots are shown corresponding to the two evaluation measures that we use, for a wide range of budget size $k$ in proportion to the number of nodes of that network. 

Firstly, we should note that the influence and the spectral radius measures correlate generally well across all reported experiments; they present similar decrease \wrt budget increase and hence `agree' in the \emph{order of effectiveness} of each policy when examined individually. 
As expected, all policies perform more comparably when very few or too many resources are available. In the former case, the very `central' nodes are highly prioritized by all methods, while in the latter the significance of node selection diminishes. Even simple approaches perform well in all but \dset{Gnutella} network where we get the most interesting results. NetShape achieves a sharp drop of the spectral radius early (\ie for small budget $k$) in \dset{Gnutella} and \dset{Epinions} networks, which drives a large influence reduction. With regards to influence minimization, the difference to competitors is bigger though in \dset{Gnutella} which is the most sparse and has the smallest strongly connected component (see \Tab{tab:networks}). In \dset{Facebook}, the reduction of the spectral radius is slower and seems less closely related with the influence, in the sense that the upper bound that we optimize is probably less tight to the behavior of the process.

Overall, the performance of the proposed NetShape algorithm is mostly as good or superior to that of the competitors, achieving up to a $50$\% decrease of the influence on the \dset{Gnutella} network compared to its best competitor.

\section{Conclusion}\label{sec:conclusion}
%
In this paper, we presented a novel framework for \emph{spectral activity shaping} under the Continuous-Time Information Cascades Model that allows the administrator for local control actions by allocating targeted resources which can alter the spread of the process. The activity shaping is achieved via the optimization of the \emph{spectral radius of the Hazard matrix} which enjoys a simple convex relaxation when used to minimize the influence of the cascade. In addition, we explained by reframing a number of use-cases that the proposed framework is general and includes tasks such as partial quarantine that acts on edges and partial node immunization that acts on nodes. Specifically for the influence minimization, we presented the \emph{NetShape} method which was compared favorably to baseline and a state-of-the-art method on real benchmark network datasets.

Among the interesting and challenging future work directions is the introduction of an `aging' feature to each piece of information to model its loss of relevance and attraction through time, and the theoretical study and experimental validation of the maximization counterpart of \emph{Netshape}. Finally, systematic experiments with random networks and time-varying node infection rates would increase our understanding on the strengths and weaknesses of this framework.

\bibliographystyle{aaai} 
\bibliography{bibl}

\end{document}